\documentclass[11pt]{article}

\usepackage[utf8]{inputenc}
\usepackage[margin=1in, a4paper]{geometry}
\usepackage[dvipsnames]{xcolor}
\usepackage[T1]{fontenc}
\usepackage{hyperref}
\hypersetup{
	colorlinks=true,
	urlcolor=MidnightBlue,
	linkcolor=MidnightBlue,
	citecolor=MidnightBlue,
	unicode
}
\usepackage{url}
\usepackage{booktabs}
\usepackage{amsfonts}
\usepackage{nicefrac}
\usepackage{microtype} 
\usepackage{xcolor}
\usepackage{amsmath}
\usepackage{amsthm}
\usepackage{mathtools}
\usepackage[shortlabels]{enumitem}
\usepackage{xspace}
\usepackage{amssymb}
\usepackage{multicol}
\usepackage[capitalise,nameinlink,noabbrev]{cleveref}
\usepackage{tikz}
\usepackage{thmtools}
\usepackage{doi}
\usepackage{thm-restate}
\usepackage{wrapfig}

\newtheorem{theorem}{Theorem}

\newtheorem{lemma}{Lemma}

\newtheorem{corollary}{Corollary}
\newtheorem{definition}{Definition}
\newtheorem{assumption}{Assumption}
\newtheorem{remark}{Remark}

\begin{document}

\begin{center}

{\huge Seeing Through Distractions: Context-Stable Attribution via the Least Core}

\vspace{9mm}

\renewcommand{\arraystretch}{1.08}
\begin{tabular}{@{}c@{\hspace{2em}}c@{\hspace{2em}}c@{}}
{\large Sai Ganesh Nagarajan}\footnotemark &
{\large Toshinori Yamauchi} &
{\large Hiroshi Kera}
\\[0.6ex]
\small\slshape University of Southern Denmark &
\small\slshape Chiba University &
\small\slshape Chiba University
\\
\small\slshape IMADA &
\small\slshape Graduate School of Informatics &
\small\slshape National Institute of Informatics
\\[0.7ex]
\small\texttt{sgnagarajan@imada.sdu.dk} &
\small\texttt{t.yamauchi@chiba-u.jp} &
\small\texttt{kera@chiba-u.jp}
\end{tabular}

\vspace{5mm}

\end{center}

\footnotetext{Corresponding author.}

\begin{quote}
\noindent\small
{\bf Abstract.}~
Shapley-value-based explanations are widely used for feature
attribution in machine learning. However, in computer vision,
a feature may reduce the target-class score when added to
the complete genuine image, even when the predicted label
remains unchanged. We formalize such features as
\emph{contextual distractors} and investigate their
attributions under semivalues and the least core.
Under a weak average marginal monotonicity assumption,
we derive quantitative semivalue--least-core gap bounds
and sufficient conditions based on the probability of
sampling smaller contexts. We further show that shifting
semivalue weights toward larger contexts decreases the
signed gap. Finally, controlled image-classification experiments and evaluations on Salient ImageNet provide strong empirical support for our theoretical findings.
\end{quote}

\section{Introduction}
A central goal in explainable AI is to understand why a classifier makes a particular decision. For a fixed input, this often amounts to identifying a subset of features that best accounts for the prediction. In image classification, for instance, one may seek a small subset of pixels, patches, or segments whose presence suffices to preserve the predicted label under masking. The size constraint is important: without it, the full image would often be a trivial but uninformative explanation. This problem becomes especially challenging in computer vision, where explanations must scale to high-dimensional inputs and, at the same time, remain meaningful.

Game-theoretic methods have proved useful in this setting because they provide principled ways to reason about feature importance. In particular, the Shapley value \cite{shapley1953value} has become a standard attribution mechanism, due to its simplicity and its appealing axiomatic properties. However, two difficulties remain. The first, which is widely recognized, is computational: exact Shapley values require evaluating the characteristic function on all subsets of the \(N\) players. As a result, much of the literature has focused on approximation methods, especially sampling-based schemes \cite{strumbelj2010efficient,lundberg2017unified,ghorbani2019data,aas2021explaining,jethani2022fastshap}, as well as more efficient computation under structural assumptions on the characteristic function \cite{fatima2008linear,castro2009polynomial}. These guarantees have been further improved under assumptions such as submodularity \cite{balkanski2015mechanisms,balkanski2017statistical} and supermodularity \cite{liben2012computing}. However, such assumptions are unlikely to hold in natural settings, especially in computer vision, so these positive guarantees generally do not extend to the regimes of interest.

While the computational challenge of Shapley values is well understood, a second and more conceptual issue that is less studied, is the effect of averaging-based attributions in the presence of spurious contextual cues. In particular, a group of features may appear beneficial on average, even though it is harmful precisely in the contexts that are most informative for the prediction. In computer vision, for example, background patches, textures, or correlated artifacts can bias the classifier toward a label without capturing the actual object of interest. Furthermore, the characteristic function that is induced from such effects is unlikely to satisfy regularities such as submodularity/supermodularity.
See Figure~\ref{fig:bias_data} for an illustration of this effect. 

This phenomenon is closely related to several well-known \emph{shortcut} examples in computer vision. In \cite{ribeiro2016should}, background snow can spuriously drive a wolf-vs-husky classifier. In \cite{beery2018recognition}, animal classifiers are shown to rely on contextual cues such as water rather than the animal itself. The Waterbirds benchmark in \cite{liu2021just} makes this explicit by contrasting waterbird/landbird identity with background, so that background can dominate the prediction. Similarly, \cite{singh2020don} documents situations in which a correlated context, such as the presence of a person together with a skateboard, can dominate the classifier's reasoning. In \cite{lapuschkin2019unmasking}, explanation maps reveal that models sometimes rely on artifacts such as source tags rather than the object.
These examples suggest that a feature group can look useful on average, yet still be harmful in the important contexts one actually cares about. This is precisely the type of behavior that Shapley values may overlook through contextual averaging.  This perspective motivates moving beyond purely averaging-based attributions toward stability-based notions such as the \emph{least-core}~\cite{maschler1979geometric}, which focus on the coalitions that matter most rather than averaging over all contexts, which is precisely the direction we pursue in this work.
%To our knowledge, this effect has not been explicitly formalized as an averaging pathology of Shapley-style explanations.

\begin{figure}[t]
    \centering
    \includegraphics[width=0.5\linewidth]{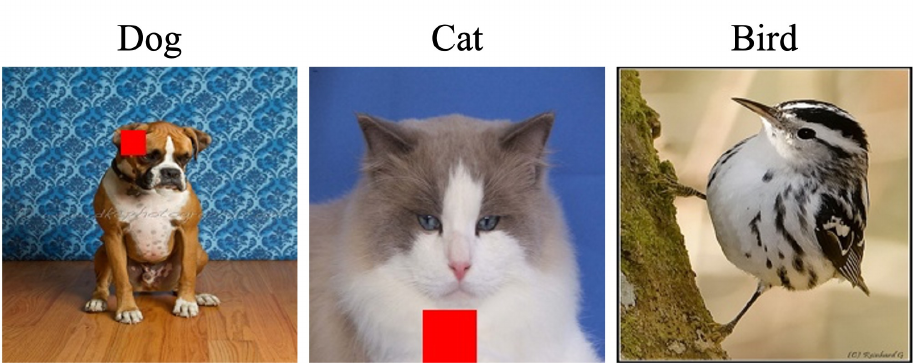}
    \caption{Controlled marker-bias setup. A red marker is added to all dog
images and a subset of cat images, while bird images remain
unmarked. This creates a dog-associated cue that may act as a
\emph{contextual distractor} for cat predictions. In
Section~\ref{sec:expts}, we observe that while Shapley by nature,
averages contributions across contexts, the least core
incorporates stability constraints that account for the
marker's effect on the complete genuine image and the predicted label.
}
    \label{fig:bias_data}
\end{figure}

\paragraph{Our Contributions.}
Our main contributions are as follows.
\begin{enumerate}

\item In Section~\ref{sec:CD}, we formalize the shortcut
phenomenon through \emph{contextual distractors}, feature
groups that reduce the target-class score when added to
the complete genuine image, even when the predicted
label remains unchanged.

\item We establish conditions under which the least core
captures this harmful effect and derive quantitative
semivalue--least-core gap bounds under a weak average
marginal monotonicity assumption. We further characterize
how the gap changes as semivalues place greater weight
on larger contexts.

\item Through controlled image-classification experiments and
evaluations on Salient ImageNet \cite{salient_imagenet}, we compare
contextual-distractor attributions across methods. We find that
Shapley and other attribution methods often assign greater
importance to distractors than the least core, even when the
genuine image content is sufficient for correct classification.
\end{enumerate}

An earlier workshop version of this work appeared at the
ICML 2026 NExT-Game Workshop~\cite{nagarajan2026seeing}.
The present paper substantially extends that preliminary version
by developing a general semivalue--least-core theory, deriving
new gap conditions and weight-shifting results under weaker
assumptions, and adding further experiments and evaluations,
including on Salient ImageNet.

\paragraph{Related Work.}
Our work is related to three strands of literature. First, Shapley-value-based explanations and, more broadly, semivalue methods such as Data Shapley, Beta Shapley, Data Banzhaf, and weighted Banzhaf have become standard tools for attribution in machine learning, including in computer vision through Shapley estimators tailored to vision transformers \cite{lundberg2017unified,sundararajan2020many,kwon2021beta,wang2023data,covert2023learning}.
Unlike this line of work, which remains averaging-based, we compare semivalue methods against a stable attribution method, which is the least-core.

Second, a growing literature studies limitations of Shapley-style explanations, including their dependence on the choice of value function, baseline values, occlusion strategy, or feature distribution, and the possibility of misleading feature-importance scores \cite{sundararajan2020many,janzing2020feature,aas2021explaining,huang2024failings,ren2022towards,blucher2022prediff,bluecher2024decoupling}.
Our work complements these critiques by identifying a specific mechanism behind such failures in vision-like settings, namely the presence of contextual distractors.

Third, our work is related to the literature on shortcut learning, spurious correlations, and explanation-based bias detection in classifiers \cite{adebayo2022post}. This work documents that models can rely on spurious cues and studies whether post hoc explanations can reveal such behavior. Our contribution is about formalizing this phenomenon through coalitional games and the notion of a contextual distractor. Furthermore, we show both theoretically and empirically when stability-based explanations can differ sharply from averaging-based ones.

Recent work has also modified Shapley-style attribution considering interactions \cite{zhang2020Interpreting,chang2025rethinking,sumiyasu2024identifying,blucher2022prediff}, 
but these approaches remain within the averaging paradigm are modified Shapley values, whereas our focus is on the least-core as an alternative principle.

Finally, with respect to the least core as a viable alternative
to Shapley-style explanations, works such as
\cite{yan2021if,gemp2024approximating} have developed computational
methods for approximating it in large games. In the setting of
contextual-distractor games, Kera et al.~\cite{kera2026rethinking}
study the complementary problem of how least-core attributions can
be computed and selected in practice. More recently,
Li et al.~\cite{li2026nonlinear} study attribution methods beyond
Shapley, including the egalitarian least core, which selects the
minimum-$\ell_2$-norm allocation from the least-core set, and
evaluate how well such methods identify important features.

Our contribution addresses a different question. We study when
context-dependent feature effects create a systematic separation
between Shapley/semivalue and least-core explanations in a broad
non-convex regime. In particular, our theoretical guarantees apply
to every allocation in the least-core set, rather than to a specific
least-core selector.

\section{Game Theory Preliminaries}
In this section, we introduce some game theory preliminaries, including the concept of the least-core and semivalues.
Consider a transferable-utility game $\mathcal{G}=(N,v)$, where $N$ denotes th set of players and $v: 2^N \rightarrow \mathbb{R}$ denotes the characteristic function or value function. A payoff vector is any vector \(x\in\mathbb{R}^{|N|}\), where \(x_i\) denotes the payoff assigned to player \(i\in N\). For any coalition \(S\subseteq N\), we write
\[
x(S):=\sum_{i\in S} x_i.
\]
When \(x(N)=v(N)\), the payoff vector is called \emph{efficient}.
See \cite{chalkiadakis2011computational} (Chapter 2) for more details.

The goal of an explanation method is precisely to distribute this total value among the features. In this view, each subset \(S\subseteq N\) defines a coalition, and its worth is given by a scalar coalition value \(v(S)\). Different explanation methods correspond to different solution concepts for performing this allocation/attribution.

\subsection{Semivalues}

Semivalues, introduced by Dubey, Neyman, and Weber
\cite{dubey1981value}, generalize the Shapley and Banzhaf values.
We consider symmetric semivalues, which assign equal weight to
coalitions of the same size, although the total weight may vary
across sizes.

\begin{theorem}[Representation of semivalues
\cite{dubey1981value}]
Let $n=|N|$. A symmetric semivalue admits a representation
\[
\phi^w(i;v)
=
\sum_{k=0}^{n-1}w_k
\frac{1}{\binom{n-1}{k}}
\sum_{\substack{S\subseteq N\setminus\{i\}\\|S|=k}}
\bigl(v(S\cup\{i\})-v(S)\bigr),
\]
where $w_k\ge 0$ and $\sum_{k=0}^{n-1}w_k=1$.
\end{theorem}

\paragraph{Shapley value.}
Setting $w_k=1/n$ for every $k=0,\ldots,n-1$
recovers the Shapley value:
\begin{equation}\label{eqn:shapval}
\phi^{\mathrm{Shap}}(i;v)
=
\frac{1}{n}
\sum_{k=0}^{n-1}
\frac{1}{\binom{n-1}{k}}
\sum_{\substack{S\subseteq N\setminus\{i\}\\|S|=k}}
\bigl(v(S\cup\{i\})-v(S)\bigr).
\end{equation}
Unlike general semivalues, the Shapley value satisfies
efficiency and is the unique efficient semivalue.
The Banzhaf value instead assigns equal weight to every
individual coalition, corresponding to
$w_k=\binom{n-1}{k}/2^{n-1}$.

\paragraph{Group semivalues.}
To attribute value to a group $D\subseteq N$, we treat it
as a single player \cite{shapley1953value,flores2019shapley}.
Let $m:=|N\setminus D|$. The resulting group semivalue is
\begin{equation}\label{eqn:grpsemival}
\phi_D^w(v)
=
\sum_{k=0}^{m}w_k
\frac{1}{\binom{m}{k}}
\sum_{\substack{T\subseteq N\setminus D\\|T|=k}}
\bigl(v(T\cup D)-v(T)\bigr),
\end{equation}
where $w_k\ge0$ and $\sum_{k=0}^{m}w_k=1$.

\paragraph{Sampling interpretation.}
The size weights $w=(w_0,\ldots,w_m)$ induce a probability
distribution over contexts $T\subseteq N\setminus D$:
\[
\mu_w(T)
=
\frac{w_{|T|}}{\binom{m}{|T|}}.
\]
Equivalently, one first samples the context size $k$
according to $w$, then samples uniformly among all
$k$-subsets of $N\setminus D$.
Writing $\Delta_D(T):=v(T\cup D)-v(T)$ and
suppressing the dependence on $v$, we obtain
\begin{equation}\label{eqn:expgrpsemival}
\phi_D^w
=
\mathbb{E}_{T\sim\mu_w}[\Delta_D(T)].
\end{equation}
This representation naturally suggests estimating semivalues
by sampling contexts and averaging their marginal contributions.
Approximation guarantees have also been studied under
additional structural assumptions, including submodularity
and related restricted settings
\cite{balkanski2017statistical,liben2012computing}.

\subsection{Least Core}

The core seeks a stable payoff allocation. For a payoff vector $x$, let
$x(S):=\sum_{i\in S}x_i$ denote the payoff assigned to coalition
$S\subseteq N$. An allocation is \emph{efficient} if $x(N)=v(N)$.
It is stable if no coalition $S$ can improve upon its allocation by
forming its own coalition, i.e., $x(S)\ge v(S)$. This gives the
definition of the core, attributed to Shapley and Gillies
\cite{shapley1955markets,gillies1959solutions}.

\begin{definition}\label{def:core}
The core of a game $\mathcal{G}=(N,v)$ is the set
\[
\mathcal{C}(\mathcal{G})
:=\{x\in\mathbb{R}^{|N|}:x(N)=v(N),\
x(S)\ge v(S)\ \forall S\subseteq N\}.
\]
\end{definition}

The core may be empty; see Example 2.22 in
\cite{chalkiadakis2011computational}. One remedy is to impose
additional structure. For supermodular (or convex) games, satisfying
$v(S\cup T)+v(S\cap T)\ge v(S)+v(T)$ for all $S,T\subseteq N$,
the core is nonempty and contains the Shapley value
\cite{Shapley1971}. Thus, the Shapley value is also stable in this
setting, although this guarantee does not extend to general games.
An alternative, introduced by \cite{maschler1979geometric}, is to
relax stability through the $\varepsilon$-core.

\begin{definition}\label{def:epscore}
For $\varepsilon\ge 0$, the $\varepsilon$-core of
$\mathcal{G}=(N,v)$ is
\[
\mathcal{C}_{\varepsilon}(\mathcal{G})
:=\{x\in\mathbb{R}^{|N|}:x(N)=v(N),\
x(S)\ge v(S)-\varepsilon\ \forall S\subseteq N\}.
\]
\end{definition}

Since the $\varepsilon$-core is nonempty for sufficiently large
$\varepsilon$, we seek the smallest relaxation admitting a feasible
allocation.

\begin{definition}\label{def:leastcore}
The \emph{least-core value} is
$\varepsilon^*(\mathcal{G})
:=\min\{\varepsilon\ge 0:
\mathcal{C}_{\varepsilon}(\mathcal{G})\ne\emptyset\}$.
The \emph{least core} is
$\mathcal{LC}(\mathcal{G})
:=\mathcal{C}_{\varepsilon^*(\mathcal{G})}(\mathcal{G})$.
\end{definition}

\paragraph{Computing the least core.}
The least-core value is obtained from the linear program
\begin{equation}\label{eqn:leastcore}
\min_{x,\varepsilon\ge 0}\ \varepsilon
\quad\text{s.t.}\quad
x(N)=v(N),\qquad
x(S)\ge v(S)-\varepsilon\quad\forall S\subseteq N.
\end{equation}
An allocation $x^*$ belongs to the least core precisely when
$(x^*,\varepsilon^*)$ is feasible for this program.

In practice, one can use approximate procedures that search for highly
violated coalitions. Recent works such as
\cite{yan2021if,gemp2024approximating} pursue this direction through
sampling-based and gradient-based schemes for approximately solving the
least-core problem. However, coalition sampling may itself affect the
resulting attribution: in contextual-distractor games,
Kera et al.~\cite{kera2026rethinking} observe that sampling can make
harmful features harder to identify, even when useful features remain
well identified.

A separate issue is that the least core is set-valued, so practical
applications require a selector to choose a particular allocation.
Kera et al.~\cite{kera2026rethinking} further show that this choice can
matter in practice and propose the \emph{entropic least core}, a smooth
selection procedure that approaches the nucleolus as its temperature
parameter decreases. Our theoretical results, however, hold for every
allocation in the least core and are therefore independent of the
choice of selector.

\section{Explanations in Classifications with Contextual Distractors}

Shapley values have been widely used in image classification to
visualize which pixels or image regions contribute most to a model's
output for a given target class.

Let $\mathcal{X}$ and $\mathcal{Y}=\{1,\ldots,K\}$ denote the image
and label spaces, respectively. An image classifier is a map
$f:\mathcal{X}\to\mathbb{R}^K$, where $f(x)$ is the vector of logits
for an image $x\in\mathcal{X}$. The predicted label is the class
with the largest logit.

To associate a cooperative game with a fixed image $x$, let $N$
denote its set of pixels (or features). For each $S\subseteq N$,
let $x_S$ denote the masked image that retains only the features
in $S$ and replaces the remaining features with a fixed baseline.
For a target class $y\in\mathcal{Y}$, define the induced value
function by
\[
    v_f^{x,y}(S):=[f(x_S)]_y.
\]
Thus, $v_f^{x,y}(S)$ is the logit corresponding to class $y$
and may be negative. The target class is fixed across all
coalitions, regardless of the predicted label of the masked image.

Importantly, the same feature may have different marginal
contributions depending on the target class being explained.
For notational simplicity, we fix $x$ and $y$ and write $v$
instead of $v_f^{x,y}$. The resulting cooperative game is
denoted by $\mathcal{G}=(N,v)$. Our theoretical results apply
to any fixed target class, although their assumptions and
the resulting attribution magnitudes may differ across classes.

\subsection{Contextual Distractors}\label{sec:CD}

In this section, we study how the contribution of a feature can depend
strongly on the surrounding image context. Importantly, this effect is
relative to the target class being explained: the same feature may
suppress the output of one class while supporting another. This is
particularly relevant for spurious features that become associated
with a class during training. Such a feature may act as a supportive
cue for the associated class, while becoming a contextual distractor
when explaining another class.

We focus on the latter phenomenon. Given a fixed target class, a
contextual distractor is a feature whose presence reduces the
corresponding model output in the complete genuine image context.
Its marginal effect may nevertheless be substantially different in
partial contexts, which creates a distinction between attribution
methods that average over contexts and core-based notions that are
constrained by particular coalitions.

\begin{definition}[Contextual distractor]
\label{def:contextual-distractor}
Fix a target class $y$ and let $D\subseteq N$ be a feature
or group of features, treated as a single player. Let
$C:=N\setminus D$ denote the remaining genuine image features.
We call $D$ a $\gamma$-contextual distractor, for $\gamma>0$, if
\[
    \Delta_D(C)
    :=v(C\cup D)-v(C)
    =v(N)-v(C)\le -\gamma.
\]
Thus, adding $D$ to the complete genuine image reduces the
target-class logit by at least $\gamma$.
\end{definition}

This definition concerns only the effect of $D$ on the complete
genuine image and imposes no restrictions on its contribution in
smaller contexts. In particular, a contextual distractor may
initially be helpful when little genuine image content is present,
but its contribution may decrease as more context is revealed,
eventually becoming harmful. More generally, its contribution may
remain negative throughout while becoming increasingly harmful.
We study this behavior through an average marginal monotonicity
assumption, which requires only that the average contribution of
$D$ is nonincreasing with context size. We observe this both in controlled experiments and on natural images from the Salient ImageNet dataset studied in Section~\ref{sec:expts}, particularly in Figures~ \ref{fig:marker_analysis} and \ref{fig:natural-image-results}, where we observe a \emph{near-linear} decrease.

\begin{assumption}[Average marginal monotonicity]
\label{ass:average-monotonicity}
Let $m=|C|$ and let $T$ be drawn uniformly at random
from $2^C$. For each $j=0,\ldots,m$, define
\[
r_j:=
\mathbb{E}\bigl[\Delta_D(T)\mid |T|=j\bigr]
=
\frac{1}{\binom{m}{j}}
\sum_{\substack{T\subseteq C\\|T|=j}}\Delta_D(T).
\]
Since $C$ is the only coalition of size $m$, we have
$r_m=\Delta_D(C)$. We assume that the average marginal
contribution of $D$ is nonincreasing with context size:
\[
r_0\ge r_1\ge\cdots\ge r_m=\Delta_D(C).
\]
\end{assumption}

We state our main theorem below and provide some explanations surrounding it. 
\begin{theorem}[Semivalue--least-core gap]
\label{thm:main}
Let $D$ be a $\gamma$-contextual distractor satisfying
average marginal monotonicity. Let $\phi_D^w$ be a
symmetric semivalue and let $J$ denote its sampled
context size, so that $\Pr(J=j)=w_j$. Define
$d_k:=r_{k-1}-r_k\ge0$ for $k=1,\ldots,m$.

Then every $x\in\mathcal{LC}$ satisfies
\begin{equation}\label{eqn:general-gap}
\phi_D^w-x(D)
\ge
\sum_{k=1}^m d_k\,\Pr_w(J<k)
-\varepsilon^\star.
\end{equation}
In particular, a gap of at least $\delta>0$ is
guaranteed if
\[
\sum_{k=1}^m d_k\,\Pr_w(J<k)
\ge\varepsilon^\star+\delta.
\]

Moreover, suppose that $r_k>\beta$ for some $k<m$, where $\beta+\gamma>0$. 
Consequently, a gap of at least $\delta>0$ is
guaranteed whenever
\[
\Pr_w(J\ge k+1)
\le
1-\frac{\varepsilon^\star+\delta}{\beta+\gamma}.
\]

\end{theorem}

The proof of this theorem is given in Appendix~\ref{app:mainthmproof}. Here, we give explain the conditions and consequences in detail. Firstly, the sufficient condition captures the tension between the
averaging behavior of semivalues and the coalitional stability
of the least core. Suppose the distractor is helpful on average
in smaller contexts ($r_k>\beta>0$), but harmful in the
complete genuine context ($r_m\le-\gamma<0$).
A gap of at least $\delta>0$ is guaranteed whenever
\[
\Pr_w(J\ge k+1)
\le
1-\frac{\varepsilon^\star+\delta}{\beta+\gamma}.
\]
Thus, when a semivalue places sufficiently little weight on
larger contexts, it may overlook the distractor's harmful
effect in the presence of the genuine image features and
assign it greater attribution than the least core.
The condition becomes easier to satisfy as the contrast
$\beta+\gamma$ increases, and harder as the least-core
value $\varepsilon^\star$ increases.

\begin{remark}[Interpretation of the attribution gap]
\label{rem:gap-interpretation}
Consider a cat image that is correctly classified both with and
without a distractor $D$, but whose cat logit decreases when
$D$ is added. Ideally, the explanation should reflect the
contribution of the genuine cat features rather than assign
excessive credit to the distractor. 

A positive Shapley--least-core gap therefore
captures the greater relative attribution assigned to the
distractor by Shapley, consistent with our experimental
observations.
\end{remark}

The following lemma characterizes how the semivalue--least-core
gap changes with context weighting. Under average marginal
monotonicity, shifting weight toward larger contexts decreases
the signed gap.

\begin{lemma}[Effect of context weighting on the gap]
\label{lem:semivalue-weight-shift}
Suppose $D$ satisfies average marginal monotonicity,
i.e., $r_0\ge r_1\ge\cdots\ge r_m$.
Consider two symmetric semivalues with context-size
distributions $w$ and $w'$. Suppose $w'$ places
stochastically greater weight on larger contexts:
\[
\sum_{j=k}^m w'_j
\ge
\sum_{j=k}^m w_j,
\qquad k=1,\ldots,m.
\]
Define their signed gaps by
\[
G_w:=\phi_D^w-\max_{x\in\mathcal{LC}}x(D),
\qquad
G_{w'}:=\phi_D^{w'}-\max_{x\in\mathcal{LC}}x(D).
\]
Then
\begin{equation}\label{eqn:weight-shift}
G_w-G_{w'}
=
\sum_{k=1}^m
(r_{k-1}-r_k)
\sum_{j=0}^{k-1}(w_j-w'_j)
\ge 0.
\end{equation}
The inequality is strict if both factors in at least
one summand are strictly positive.
\end{lemma}
The proof of this lemma is stated in Appendix~\ref{app:lemmaweightshift}. Experimental evalutions shown in Appendix~\ref{sec:group-semivalue-experiment} confirm this pattern, in particular Figure~\ref{fig:group-semivalue-results}. We next discuss the limitations of semivalues that concentrate their weights on either very large or very small contexts.

\begin{remark}[Limitations of extreme context weighting]
Although placing greater weight on larger contexts reduces the
semivalue--least-core gap, this need not yield a more informative
explanation. A feature may genuinely support the target class
in smaller contexts, but become less helpful or harmful as more
features are revealed. Semivalues that concentrate excessively
on larger contexts may overlook this positive contribution,
just as those emphasizing smaller contexts may overlook its
harmful effect in the complete image. The least core, in
contrast, considers all coalitions through its stability
constraints rather than assigning fixed weights to different
context sizes.
\end{remark}

\paragraph{Specializations.}
Our results apply directly to Shapley, weighted Banzhaf,
and Beta-Shapley values. For Shapley, the single-threshold
condition in Theorem~\ref{thm:main} simplifies to

\[
\Pr_{\mathrm{Sh}}(J\ge k+1)
=\frac{m-k}{m+1}
\le 1-\frac{\varepsilon^\star+\delta}{\beta+\gamma}.
\]
Furthermore, Lemma~\ref{lem:semivalue-weight-shift}
shows that the signed gap is nonincreasing in the
weighted Banzhaf parameter $p$, and nondecreasing
in $a$ and nonincreasing in $b$ for Beta-Shapley.
The corresponding corollaries and proofs are given
in Appendix~\ref{app:addcor}.

\section{Experiments}\label{sec:expts}

We investigate contextual distractors in two settings:
controlled image-classification experiments and evaluations
on Salient ImageNet. We compare their effects on Shapley
and least-core attributions, with additional experiments
reported in Appendix~\ref{app:addexpts}.

\subsection{Controlled Marker-bias experiments}

First, we conduct a controlled marker-bias experiment to test whether the theoretical Shapley--least-core gap appears in image classification. 

\begin{figure*}[ht]
    \centering
    \includegraphics[width=\textwidth]{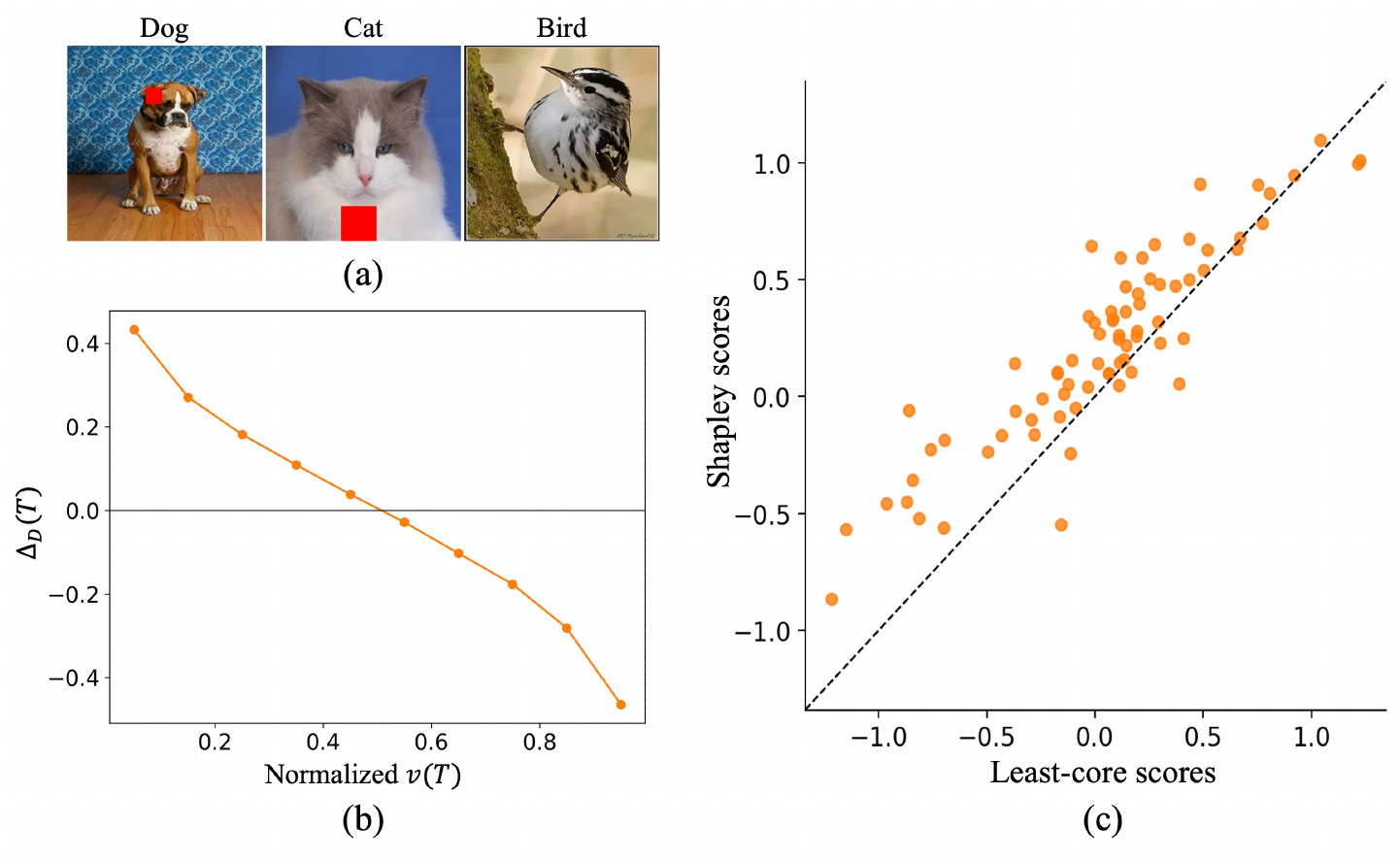}
    \caption{
    Analysis of the controlled marker-bias experiment.
    (a) Example images with synthetic markers.
    (b) For 100 cat-class images, the x-axis shows the normalized cat-class logit $v(T)$ of masked images, where coalitions $T$ are sampled from non-marker regions. The y-axis shows the logit change caused by adding the marker region $D$, i.e., $v(T \cup D) - v(T)$.
    (c) Comparison of Shapley and least-core scores for the marker region over the same 100 cat-class images. Each point corresponds to one image, and the dashed line indicates equal attribution by the two methods.
}
    \label{fig:marker_analysis}
\end{figure*}

We construct a three-class dog/cat/bird task and train a ResNet-50~\cite{resnet50} classifier. 
As shown in Fig.~\ref{fig:marker_analysis}(a), 
a red marker region $D$ is added to $100\%$ of dog images and $30\%$ of cat images, 
while bird images have no marker. 
This makes $D$ strongly associated with dog, 
so that it can act as a misleading contextual cue for cat images.

We first verify that $D$ behaves as a contextual distractor for cat predictions.
Each image is divided into a grid of patches, 
which define the player set $N$.
The marker region $D\subseteq N$ is represented by the set of patches overlapping the red marker.
We then sample coalitions $T\subseteq N\setminus D$ from the non-marker patches and compute $\Delta_D(T)=v(T\cup D)-v(T)$.
As shown in Fig.~\ref{fig:marker_analysis}(b), 
$\Delta_D(T)$ decreases as the context becomes stronger and changes from positive to negative.
Thus, the marker helps cat prediction in weak contexts but suppresses it when strong cat evidence is already present.

Finally, we compare the attribution assigned to the marker region.
Images ($224\times224$) are divided into an $8\times 8$ grid, yielding 64 patch players. 
Shapley values are estimated by Monte Carlo permutation sampling using 100 random permutations per image.
For each permutation, we sequentially add the 64 patches and record the change in the model output, resulting in $100 \times 64$ sampled output changes per image.
For least-core, we avoid enumerating all coalitions and instead solve the linear program on sampled constraints~\cite{yan2021if}. 
We sample 6,400 coalitions by independently including each player with probability $p=0.5$, and add the empty and grand coalitions.
Optimization is implemented with CVXPY. 
Neural-network evaluations are performed with PyTorch on CUDA, while the CVXPY optimization is typically run on CPU. 
Figure~\ref{fig:marker_analysis}(c) shows that Shapley assigns larger values to $D$ than the least-core for most images.
This supports our theoretical prediction that Shapley can over-credit a contextual distractor by averaging over contexts, 
whereas least-core penalizes features that destabilize high-value coalitions.

\subsection{Contextual Distractors in Natural Images}
\label{sec:natural-contextual-distractors}

We next test whether the same separation occurs for naturally occurring
regions rather than synthetic markers.  We use class--feature pairs from
Salient ImageNet~\cite{salient_imagenet} to identify candidate spurious
regions in ImageNet training images.  Each image is represented as a
$7\times7$ patch game, and the target-class logit is used as the coalition
value.  Candidate regions are independently localized using Grounding DINO~\cite{grounding_dino} and SAM~\cite{sam}, 
and are retained only when the resulting region $D$ satisfies the
contextual-distractor sign reversal: its mean marginal contribution is
positive in weak contexts and negative in strong contexts, with the expected
sign observed in at least $60\%$ of sampled contexts in both groups. 
We analyze 354 contextual distractors for ResNet-50 and 108 for ViT-B/16,
considering only images for which the model prediction matches the ground-truth label.
Full selection and computation details are given in
Appendix~\ref{app:natural-image-details}. Further evaluations are shown in Appendix~\ref{app:natural-contextual-distractors}.

For each candidate, we compare Monte Carlo Shapley with a sampled
egalitarian least-core allocation (\emph{Core}, see Appendix~\ref{app:natural-core-computation} for details), reporting total
distractor attribution $A_D=\sum_{i\in D}a_i$ and mask-weighted
attribution $A_{\mathrm{mask}}=\sum_{i=1}^{49}a_iw_i$, where $w_i$
is the fraction of patch $i$ covered by the final mask.
As controls, we generate 200 prediction-active random-SAM
regions per architecture, excluding those, where the contextual-distractors are harmful enough that they change the label.

\begin{figure*}[t]
    \centering
    \includegraphics[width=\textwidth]{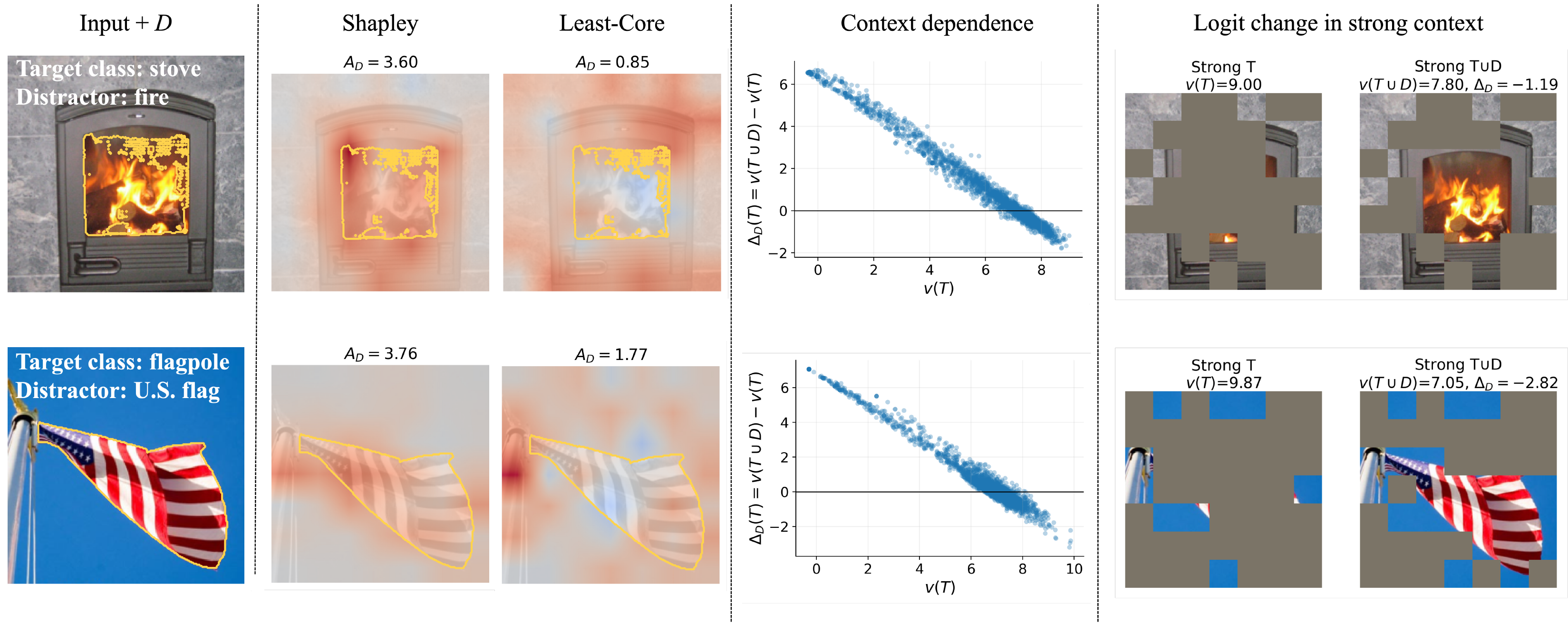}
    \caption{Representative examples of contextual distractors.
    \textbf{Input + D:} the original input, with the distractor region $D$
    outlined in yellow.  \textbf{Shapley and Least-Core:} patch-attribution
    heatmaps, where red and blue indicate positive and negative attribution,
    respectively; $A_D$ gives the total attribution assigned to the patches
    in $D$.  \textbf{Context dependence:} the marginal contribution
    $\Delta_D(T)=v(T\cup D)-v(T)$ plotted against the target-class logit
    $v(T)$ obtained from the surrounding context $T$.  \textbf{Logit change
    in strong context:} a sampled strong coalition $T$ before and after adding
    $D$, reporting the resulting target-class logit $v(T\cup D)$ and its
    change $\Delta_D$ from $v(T)$.}
    \label{fig:natural-image-results}
\end{figure*}

\paragraph{Qualitative evaluation.}
Figure~\ref{fig:natural-image-results} illustrates the same mechanism with two
natural-image examples.  In the stove image, the fire acts as the distractor
$D$, whereas in the flagpole image, the U.S. flag serves as $D$.
In both examples, the context-dependence plots show that
$\Delta_D(T)$ decreases as the target-class logit $v(T)$ of the surrounding
context increases, changing from positive in weak contexts to negative in
strong contexts.  Thus, the fire and the flag do not provide uniformly useful
class evidence; instead, their contributions depend on the surrounding
context, consistent with the behavior of contextual distractors.

The strong-context examples further illustrate this effect.
For the stove image, the surrounding context alone already yields a high
stove logit of $v(T)=9.00$.  Adding the fire decreases the logit to
$v(T\cup D)=7.80$, giving $\Delta_D=-1.19$.
This example shows that the fire is not necessary to sustain a high stove
score in this context; rather, other visual evidence, such as the stove frame,
is sufficient and remains informative.
Nevertheless, Shapley assigns a large attribution of $A_D=3.60$ to the fire,
whereas Core assigns only $0.85$ and places relatively greater emphasis on
the stove structure.
A similar pattern appears in the flagpole example.
The pole provides stable evidence for the target class, whereas adding the
U.S. flag to the displayed strong context decreases the flagpole logit from
$9.87$ to $7.05$, yielding $\Delta_D=-2.82$.
Despite this negative marginal contribution, Shapley assigns $A_D=3.76$ to
the flag, compared with $1.77$ under Core, which places relatively greater
emphasis on the pole.

These examples show how Shapley heatmaps may overemphasize context-dependent cues such as fire or flags, even in the presence of stronger contexts. The least core assigns them less credit, placing greater emphasis on evidence that remains useful.

\section{Limitations and Discussion}\label{sec:lim}

Our results highlight the least core as a complementary attribution
method, particularly for computer vision, where a large
semivalue--least-core gap may indicate shortcut features or
contextual distractors. Its main limitation is computational:
Eqn.~\eqref{eqn:leastcore} involves exponentially many coalition
constraints. Although the LP can be solved in polynomial time
given an efficient separation oracle via the ellipsoid method
\cite{grotschel2012geometric,fang2016computing}, such oracles
may be difficult to obtain in practice. Developing scalable
least-core algorithms, through application-specific structure
or heuristic separation oracles, remains an important direction
for future work.

We conclude with the perspective of Hart and Mas-Colell
\cite{hart1988potential}, who emphasize the value of studying
different cooperative solution concepts:
\begin{center}
\emph{A simpler solution may be easier to study and apply, which
increases its usefulness. However, it is important to look at
different solution concepts, based on different postulates,
because they illuminate the problem from different angles.}
\end{center}
In this spirit, the least core offers a complementary perspective
to the Shapley value by emphasizing coalitional stability
rather than average marginal contributions.

\section*{Acknowledgments}
This research was partially supported by JST PRESTO Grant Number JPMJPR24K4,
JST BOOST Program Grant Number JPMJBY24C6, 
JSPS KAKENHI Grant Number JP23KK0208 and 26K02996,
and MEXT Supporting Pioneering Research through AI for 1,000 Discovery challenges Program (SPReAD) Japan Grant Number JPMXP1726301424.

\small 
\DeclareUrlCommand{\Doi}{\urlstyle{sf}} 
\renewcommand{\path}[1]{\small\Doi{#1}} 
\renewcommand{\url}[1]{\href{#1}{\small\Doi{#1}}} 
\bibliographystyle{alphaurl} 
\bibliography{bibliography.bib}

\appendix

\section{Missing Proofs and Additional Corollaries}

\subsection{Proof of Theorem~\ref{thm:main}}\label{app:mainthmproof}

\begin{proof}
Let $x\in\mathcal{LC}$ be any least-core allocation.
By efficiency and the least-core constraint for $C$,
\[
x(D)=v(N)-x(C)
\le v(N)-v(C)+\varepsilon^\star
=r_m+\varepsilon^\star.
\]
Since $\phi_D^w=\sum_{j=0}^m w_jr_j$ and
$\sum_{j=0}^m w_j=1$, we obtain
\[
\phi_D^w-x(D)
\ge
\sum_{j=0}^{m-1}w_j(r_j-r_m)
-\varepsilon^\star.
\]
Writing $d_\ell:=r_{\ell-1}-r_\ell\ge0$,
average marginal monotonicity gives
\[
r_j-r_m=\sum_{\ell=j+1}^m d_\ell.
\]
Exchanging the order of summation yields
\[
\phi_D^w-x(D)
\ge
\sum_{\ell=1}^m
d_\ell\,\Pr_w(J<\ell)-\varepsilon^\star.
\]
This establishes the general gap bound.

Now suppose that $r_k>\beta>0$ for some $k<m$.
By average marginal monotonicity,
$r_j\ge r_k>\beta$ for every $j\le k$.
Since $D$ is a $\gamma$-contextual distractor,
$r_m\le-\gamma$, and hence
\[
r_j-r_m>\beta+\gamma
\qquad\text{for every }j\le k.
\]
For $j>k$, monotonicity still guarantees
$r_j-r_m\ge0$. Therefore,
\[
\begin{aligned}
\phi_D^w-x(D)
&\ge \sum_{j=0}^{m-1}w_j(r_j-r_m)
      -\varepsilon^\star\\
&\ge (\beta+\gamma)\Pr_w(J\le k)
      -\varepsilon^\star\\
&= (\beta+\gamma)
   \bigl(1-\Pr_w(J\ge k+1)\bigr)
   -\varepsilon^\star.
\end{aligned}
\]
Consequently, if
\[
\Pr_w(J\ge k+1)
\le
1-\frac{\varepsilon^\star+\delta}{\beta+\gamma},
\]
then $\phi_D^w-x(D)\ge\delta$ for every
$x\in\mathcal{LC}$.

\end{proof}

\subsection{Proof of Lemma~\ref{lem:semivalue-weight-shift}}\label{app:lemmaweightshift}
\begin{proof}
Let $d_k:=r_{k-1}-r_k\ge0$. For every
$j=0,\ldots,m$, we have
\[
r_j=r_m+\sum_{k=j+1}^m d_k.
\]
Using $\sum_{j=0}^m w_j=1$ and exchanging
the order of summation gives
\[
\begin{aligned}
\phi_D^w
&=\sum_{j=0}^m w_jr_j\\
&=r_m+\sum_{k=1}^m d_k
  \sum_{j=0}^{k-1}w_j.
\end{aligned}
\]
The least-core term is identical for both
semivalues, since the underlying game is fixed.
Subtracting the analogous expression for $w'$
therefore yields
\[
G_w-G_{w'}
=
\sum_{k=1}^m d_k
\left(
\sum_{j=0}^{k-1}w_j
-
\sum_{j=0}^{k-1}w'_j
\right).
\]
By stochastic dominance and the fact that
both distributions sum to one,
\[
\sum_{j=0}^{k-1}w_j
-
\sum_{j=0}^{k-1}w'_j
=
\sum_{j=k}^m w'_j
-
\sum_{j=k}^m w_j
\ge0.
\]
Since $d_k\ge0$, every summand is
nonnegative, proving the result.
\end{proof}

\subsection{Consequences for classical semivalues}
\label{app:addcor}

The general gap theorem gives explicit conditions for
standard semivalues by substituting their context-size
distributions. The weight-shifting lemma additionally
characterizes how the gap changes when their parameters
are varied.

\begin{corollary}[Shapley specialization]
\label{cor:shapley-gap}
Under the hypotheses of Theorem~\ref{thm:main},
the grouped Shapley value satisfies
\begin{equation}\label{eqn:shapley-gap}
\begin{aligned}
\phi_D^{\mathrm{Sh}}-x(D)
&\ge
\frac{1}{m+1}\sum_{j=0}^{m-1}(r_j-r_m)
-\varepsilon^\star\\
&=
\frac{1}{m+1}\sum_{k=1}^m
k(r_{k-1}-r_k)-\varepsilon^\star
\end{aligned}
\end{equation}
for every $x\in\mathcal{LC}$.

Moreover, suppose that $r_k-r_m\ge\eta>0$
for some $k<m$. A gap of at least $\delta>0$
is guaranteed whenever
\[
\Pr_{\mathrm{Sh}}(J\ge k+1)
=\frac{m-k}{m+1}
\le
1-\frac{\varepsilon^\star+\delta}{\eta}.
\]
In particular, if $r_k>\beta>0$ and
$r_m\le-\gamma$, it suffices that
\[
\frac{m-k}{m+1}
\le
1-\frac{\varepsilon^\star+\delta}{\beta+\gamma}.
\]
\end{corollary}

\begin{proof}
For Shapley, $w_j=1/(m+1)$ for every $j$.
Substituting these weights into
Theorem~\ref{thm:main} gives the first inequality
in \eqref{eqn:shapley-gap}. The equality follows
from
\[
r_j-r_m
=\sum_{\ell=j+1}^m(r_{\ell-1}-r_\ell)
\]
by exchanging the order of summation.

Now suppose $r_k-r_m\ge\eta>0$.
By average marginal monotonicity,
$r_j-r_m\ge\eta$ for every $j\le k$,
while $r_j-r_m\ge0$ for $j>k$.
Consequently,
\[
\begin{aligned}
\phi_D^{\mathrm{Sh}}-x(D)
&\ge
\eta\Pr_{\mathrm{Sh}}(J\le k)
-\varepsilon^\star\\
&=
\eta\bigl(1-\Pr_{\mathrm{Sh}}(J\ge k+1)\bigr)
-\varepsilon^\star.
\end{aligned}
\]
Since Shapley samples context sizes uniformly,
\[
\Pr_{\mathrm{Sh}}(J\ge k+1)
=\frac{m-k}{m+1}.
\]
The stated sufficient condition follows.
Finally, $r_k>\beta$ and $r_m\le-\gamma$
imply $r_k-r_m>\beta+\gamma$, giving
the last condition.
\end{proof}

\begin{corollary}[Weighted Banzhaf specialization]
\label{cor:wbanz}
Under the hypotheses of Theorem~\ref{thm:main},
consider the $p$-weighted Banzhaf semivalue, with
\[
w_j(p)=\binom mj p^j(1-p)^{m-j},
\qquad 0<p<1.
\]
For every $x\in\mathcal{LC}$,
\begin{equation}\label{eqn:wbanz-gap}
\phi_D^{\mathrm{Banz}(p)}-x(D)
\ge
\sum_{j=0}^{m-1}
\binom mj p^j(1-p)^{m-j}(r_j-r_m)
-\varepsilon^\star.
\end{equation}
Furthermore:

\begin{enumerate}
\item[(a)] \textbf{Monotonicity in $p$.}
The actual signed gap
$G_{\mathrm{Banz}(p)}$ is nonincreasing in $p$.

\item[(b)] \textbf{Concentration on smaller contexts.}
If $p\le 1/(m+1)$, then
\[
G_{\mathrm{Banz}(p)}\ge G_{\mathrm{Sh}}.
\]
Thus, any sufficient condition establishing a
$\delta$-gap for Shapley also establishes a
$\delta$-gap for these weighted Banzhaf values.
\end{enumerate}
Standard Banzhaf corresponds to $p=1/2$.
\end{corollary}

\begin{proof}
The bound follows from Theorem~\ref{thm:main}
by substituting the binomial size weights.

For (a), differentiating the attribution gives
\[
\frac{d}{dp}\phi_D^{\mathrm{Banz}(p)}
=
m\sum_{j=0}^{m-1}
\binom{m-1}{j}p^j(1-p)^{m-1-j}
(r_{j+1}-r_j)\le0.
\]
Since the least-core term does not depend on $p$,
the signed gap is also nonincreasing.

For (b), when $p\le1/(m+1)$, the sequence
$j\mapsto w_j(p)$ is nonincreasing because
\[
\frac{w_{j+1}(p)}{w_j(p)}
=
\frac{m-j}{j+1}\frac{p}{1-p}\le1.
\]
A nonincreasing probability distribution on
$\{0,\ldots,m\}$ is stochastically dominated
by the uniform distribution on these sizes.
Lemma~\ref{lem:semivalue-weight-shift} therefore
gives the desired comparison with Shapley.
\end{proof}

\begin{corollary}[Monotonicity for Beta-Shapley]
\label{cor:betashap-monotone}
Under the hypotheses of Theorem~\ref{thm:main},
consider Beta-Shapley with parameters $a,b>0$
and size weights
\[
w_j^{(a,b)}
=
\binom mj
\frac{\mathrm{Beta}(j+b,m-j+a)}
     {\mathrm{Beta}(b,a)},
\qquad j=0,\ldots,m.
\]
Its signed gap satisfies
\[
G_{a,b}
\ge
\sum_{j=0}^{m-1}
w_j^{(a,b)}(r_j-r_m)
-\varepsilon^\star.
\]
Moreover:
\begin{enumerate}
\item[(a)] For fixed $b$, $G_{a,b}$ is
nondecreasing in $a$.
\item[(b)] For fixed $a$, $G_{a,b}$ is
nonincreasing in $b$.
\end{enumerate}
In particular, increasing the relative weight
on larger contexts decreases the signed gap.
Shapley is recovered when $a=b=1$.
\end{corollary}

\begin{proof}
The gap bound is an immediate application of
Theorem~\ref{thm:main}.

For the monotonicity statements, use the
beta-binomial representation
\[
\Theta_{a,b}\sim\mathrm{Beta}(b,a),
\qquad
J_{a,b}\mid\Theta_{a,b}=\theta
\sim\mathrm{Bin}(m,\theta).
\]
For fixed $b$, increasing $a$ shifts $\Theta_{a,b}$
stochastically downward and hence shifts
$J_{a,b}$ toward smaller context sizes.
For fixed $a$, increasing $b$ shifts both
distributions stochastically upward.
Both conclusions now follow from
Lemma~\ref{lem:semivalue-weight-shift}.
\end{proof}

\section{Additional Experimental Results}\label{app:addexpts}

\subsection{Epsilon evaluation for explanation methods}\label{app:epseval}
We evaluate explanation methods using a ResNet-50 classifier~\cite{resnet50} on 100 ImageNet test images.
Each image is resized to $224\times224$ and divided into an $8\times8$ grid, yielding 64 patch players.

For a fixed attribution vector $x$, we define the epsilon value as the worst coalition deficit:
$$
\epsilon(x)=\max_{S\subseteq N}\left(v(S)-x(S)\right),
$$
where $x(S)=\sum_{i\in S}x_i$. In practice, this maximum is approximated over 4,096 sampled coalitions.

For Grad-CAM~\cite{gradcam} and Grad-CAM++~\cite{gradcam++}, the explanation methods do not directly produce an epsilon value.
We therefore convert each heat map into a patch-level allocation $x$ by averaging the heat map within each patch.
Since these heat maps do not necessarily satisfy the efficiency constraint, we rescale $x$ so that
$$
\sum_i x_i = v(N).
$$
We then evaluate $\epsilon(x)$ by sampled coalitions.

For Monte Carlo Shapley, we compute $\epsilon(x)$ as a post-hoc evaluation of the resulting attribution vector.
Shapley values are estimated by Monte Carlo permutation sampling using 64 random patch orderings per image $64 \times 64 = 4,096$ sampled output changes per image).

For least-core, we use the same 64 patch players and solve the least-core optimization over 4,096 sampled coalition constraints.
We directly report the optimized $\epsilon$ returned by the least-core solver.
Thus, the least-core epsilon is not a post-hoc evaluation of a generated heat map, but the objective value obtained during the least-core optimization.

Lower epsilon indicates that the explanation better satisfies the sampled core constraints and is therefore more stable from a coalitional perspective.
Table~\ref{tab:imagenet_epsilon} summarizes the results.

\begin{table}[!ht]
    \centering
    \caption{
    Epsilon values on 100 ImageNet test images. Lower is better.
    }
    \label{tab:imagenet_epsilon}
    \begin{tabular}{lcccc}
        \toprule
         & Grad-CAM & Grad-CAM++ & Shapley-MC & Least-Core \\
        \midrule
        $\epsilon$ (mean $\pm$ std) 
        & $0.6990 \pm 0.1445$ 
        & $0.6899 \pm 0.1397$ 
        & $0.6851 \pm 0.1637$ 
        & $\mathbf{0.5822 \pm 0.1389}$ \\
        \bottomrule
    \end{tabular}
\end{table}
%%%%%%%%%%%%%%%%%%%%%%%%%

\subsection{Controlled marker experiments}

\subsubsection{Main attribution comparison}

We evaluate the contextual-distractor effect using the controlled
dog/cat/bird task in Figure~\ref{fig:marker_data_results}.  A red marker is
added to all dog images and to $30\%$ of cat images, while bird images remain
unmarked.  We train ResNet-50 and ViT-B/16 classifiers and use the cat-class
logit as the coalition value.  Images are divided into an $8\times8$ grid, and
$D$ consists of the patches overlapping the marker.

\begin{figure*}[t]
    \centering
    \includegraphics[width=\textwidth]{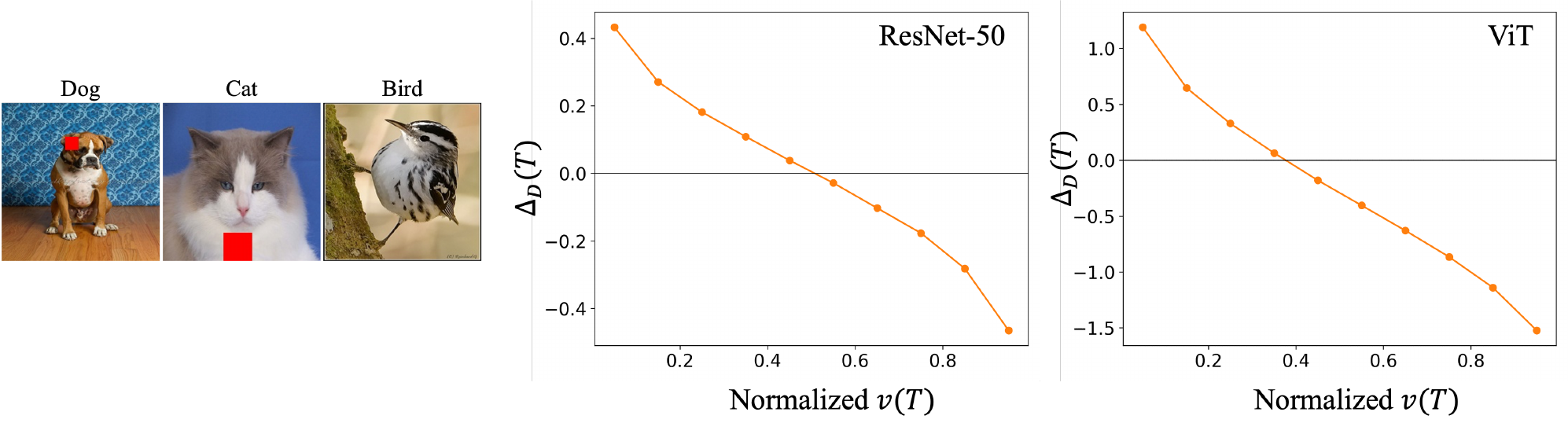}
    \caption{Controlled marker-bias experiment.  Left: representative inputs.
    The marker is correlated with the dog class and can therefore mislead cat
    predictions.  Middle and right: for both ResNet-50 and ViT-B/16, the
    marginal contribution $\Delta_D(T)=v(T\cup D)-v(T)$ decreases with the
    strength of the non-marker context.  The marker helps in weak contexts but
    suppresses the cat logit in strong contexts.}
    \label{fig:marker_data_results}
\end{figure*}

We characterize the contextual behavior of $D$ on correctly classified,
marker-bearing cat images.  For each image, we sample 2,000 coalitions
$T\subseteq N\setminus D$ by including each non-marker patch independently
with probability $0.5$, and additionally evaluate the full non-marker context.
We retain 100 images per backbone for which the mean marginal contribution
$\Delta_D(T)=v(T\cup D)-v(T)$ is positive in the bottom $20\%$ of contexts
ranked by $v(T)$ and negative in the top $20\%$.  On these selected cohorts,
the mean low- and high-context contributions are $0.351$ and $-0.373$ for
ResNet-50, and $0.917$ and $-1.331$ for ViT-B/16.  Thus, the marker helps cat
prediction when little cat evidence is visible, but suppresses it once the
surrounding evidence becomes strong.

We compare the attribution assigned to the marker by Shapley, Banzhaf, and the
least-core.  Shapley is estimated using 100 random permutations per image,
giving 100 marginal samples for each of the 64 patches.  Banzhaf likewise uses
100 marginal samples per patch.  Its context-size weights are
\[
\alpha_k(p)=\binom{|N|-1}{k}p^k(1-p)^{|N|-1-k},
\qquad k=0,\ldots,|N|-1.
\]
Larger $p$ places more weight on larger contexts.  We use $p=0.5$, which
includes each other patch independently with equal probability and recovers
the standard Banzhaf value.  For the least-core, we sample 6,400 constraints by
including each patch independently with probability $0.5$, add the empty and
grand coalitions, and solve the resulting linear program with CVXPY.  Model
evaluations use PyTorch on CUDA and the linear program is solved on CPU.  We
report the mean attribution over patches with at least $40\%$ of their area
overlapping the marker.  This yields 74 ResNet-50 images and 80 ViT-B/16
images.

\begin{figure*}[t]
    \centering
    \includegraphics[width=\textwidth]{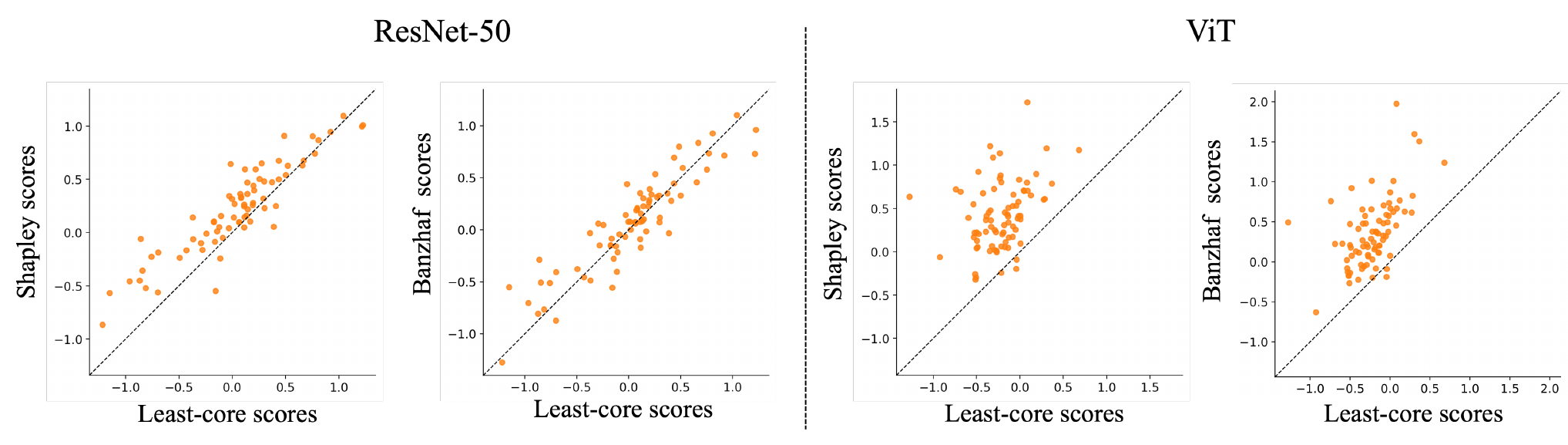}
    \caption{Per-image marker attribution for ResNet-50 (left) and ViT-B/16
    (right).  We compare the Shapley and Banzhaf ($p=0.5$) scores with the
    corresponding least-core scores.  Each point represents one retained
    image, and the dashed line indicates equal attribution; points above the
    line receive more attribution from the semivalue than from the least-core.}
    \label{fig:marker_attribution_results}
\end{figure*}

\begin{table}[t]
    \centering
    \caption{Mean marker-patch attribution. Each Gap column reports the
    corresponding semivalue minus the least-core attribution; a positive value
    indicates that the semivalue assigns more credit to the marker.}
    \label{tab:marker_semivalue_results}
    \begin{tabular}{lr|rr|rr}
        \toprule
        Model & Least-core & Shapley & Gap & Banzhaf ($p=0.5$) & Gap \\
        \midrule
        ResNet-50 & 0.041 & 0.220 & 0.179 & 0.076 & 0.035 \\
        ViT-B/16  & -0.221 & 0.419 & 0.639 & 0.355 & 0.576 \\
        \bottomrule
    \end{tabular}
\end{table}

Figure~\ref{fig:marker_attribution_results} shows the per-image comparisons,
and Table~\ref{tab:marker_semivalue_results} summarizes their means.  Both
semivalues assign more attribution to the distractor than the least-core on
average.  The separation is pronounced for ViT-B/16 and is also observed for
Shapley on ResNet-50, whereas the ResNet-50 Banzhaf gap is smaller.  These
results support the predicted averaging effect: a marker that appears useful
across sampled contexts can receive positive average credit even though it is
harmful when the surrounding cat evidence is strong.  The smaller Banzhaf gap
for ResNet-50 may partly reflect the different context-size distributions
sampled by Banzhaf and Shapley; we examine the effect of this sampling
difference in the following section.

\subsubsection{Effect of context weighting on the group semivalue gap}
\label{sec:group-semivalue-experiment}

We next examine how the distribution over context sizes affects the
semivalue--least-core gap.  This experiment treats the entire marker region
$D$ as a single player and directly estimates its group marginal
\[
    \Delta_D(T)=v(T\cup D)-v(T), \qquad T\subseteq N\setminus D.
\]
This is distinct from summing the individual semivalue attributions of the
marker patches.  For a context-size distribution $\alpha$, we estimate
\[
    \phi_\alpha^{\mathrm{group}}(D)
    =\mathbb{E}_{T\sim\mu_\alpha}[\Delta_D(T)]
\]
and compare it with $x^*(D)=\sum_{i\in D}x_i^*$ from the least-core solution of
the original patch game.  We report the signed gap
\[
    G_\alpha:=\phi_\alpha^{\mathrm{group}}(D)-x^*(D),
\]
so that $G_\alpha>0$ means that the group semivalue assigns more credit to the
marker than the least-core does.

\paragraph{Context-weighting families.}
We use the same 100 marker-bearing cat images per backbone as in the preceding
experiment and vary the context-size distribution while leaving the model,
masking procedure, baseline, and least-core allocation fixed.  For the
$p$-weighted Banzhaf family, each patch in $N\setminus D$ is included
independently with probability
\[
    p\in\{0.2,0.35,0.5,0.65,0.8\}.
\]
Thus, larger $p$ shifts probability mass toward larger contexts, while
$p=0.5$ gives the standard Banzhaf value.  We also consider the Beta-Shapley
family with
\[
    (a,b)\in\{(4,1),(2,1),(1,1),(1,2),(1,4)\},
\]
whose expected normalized context size is $b/(a+b)$; $(a,b)=(1,1)$ recovers
the group Shapley value.  As a reference baseline, we also estimate the group
Shapley value independently rather than identifying it only with the
$(a,b)=(1,1)$ Beta-Shapley setting.  Each group semivalue, including this
group Shapley baseline, is estimated from 3,200 marginal samples per image.
Since one marginal sample evaluates both $T$ and $T\cup D$, this matches the
nominal 6,400 coalition-evaluation budget used for the least-core constraints.

\begin{figure*}[t]
    \centering
    \includegraphics[width=\textwidth]{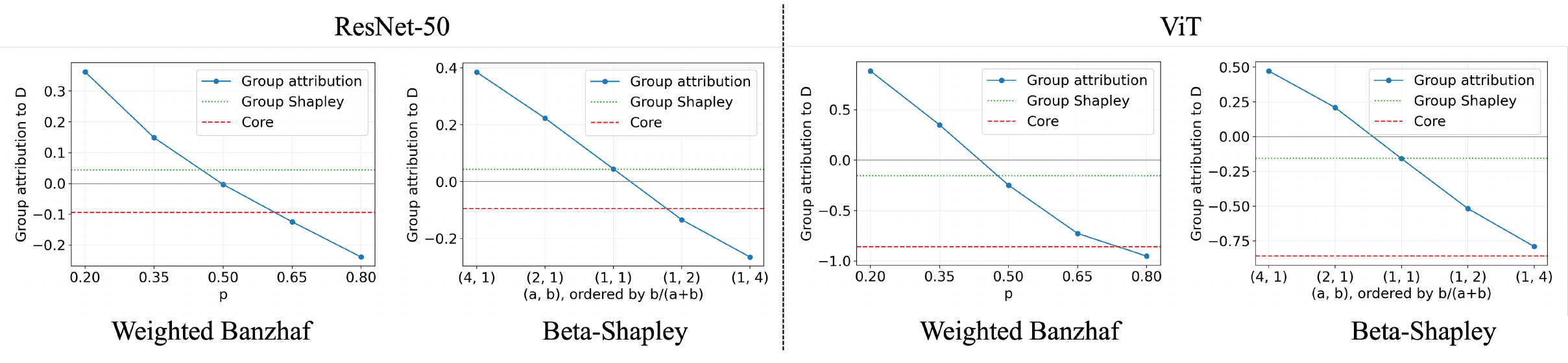}
    \caption{Effect of context-size weighting on the group attribution to the
    marker region $D$.  The left two panels show ResNet-50 and the right two
    show ViT-B/16.  Within each pair, the first panel varies $p$ for the
    weighted-Banzhaf family, and the second shifts the Beta-Shapley parameters
    from small-context weighting $(4,1)$ to large-context weighting $(1,4)$.
    The solid blue curves show the group semivalue obtained at each parameter
    setting.  The dotted green and dashed red lines show the independently
    estimated group Shapley baseline and the least-core attribution,
    respectively.}
    \label{fig:group-semivalue-results}
\end{figure*}

\paragraph{Results.}
As shown in Figure~\ref{fig:group-semivalue-results}, increasing the weight on
large contexts consistently lowers the signed gap.  For ResNet-50, the mean
weighted-Banzhaf gap decreases from $0.457$ at $p=0.2$ to $0.243$, $0.091$,
$-0.030$, and $-0.145$ as $p$ increases.  For ViT-B/16, it decreases from
$1.742$ to $1.208$, $0.610$, $0.132$, and $-0.093$.  The Beta-Shapley sweep
shows the same behavior: moving from $(a,b)=(4,1)$ to $(1,4)$ changes the gap
from $0.479$ to $-0.171$ for ResNet-50 and from $1.331$ to $0.069$ for
ViT-B/16.  The independently estimated group Shapley gaps are $0.138$ and
$0.703$, respectively, consistent with the corresponding $(1,1)$
Beta-Shapley estimates.

These results agree with the mechanism in Lemma~\ref{lem:semivalue-weight-shift}, i.e., placing
mass toward larger contexts increases the chance of evaluating $D$ where it
is harmful.  Its group attribution therefore decreases toward the
least-core attribution, reducing the positive over-credit gap.  At the
largest weights the signed gap can cross zero, so the result should not be
read as saying that the absolute discrepancy must decrease indefinitely;
rather, increasing large-context weight systematically removes the positive
semivalue over-credit predicted for contextual distractors.

\subsection{Contextual Distractors in Natural Images}
\label{app:natural-contextual-distractors}

We provide additional evaluations for the experiments described in Section~\ref{sec:expts} for the Salient ImageNet dataset.

\paragraph{Quantitative evaluation.}
Table~\ref{tab:natural-shapley-core} shows a clear contrast between the two
region types.  
On contextual distractors, Shapley assigns systematically more
credit than Core: the mean $A_D$ gap is 1.076 for ResNet-50 and 0.532 for
ViT-B/16, 
with positive gaps in $82.5\%$ and $80.6\%$ of candidates,
respectively. 
The conclusion is unchanged under mask weighting.  
In contrast, the random-SAM controls show the opposite tendency:
the Core attribution is larger than the Shapley attribution on average for
both architectures and for both $A_D$ and $A_{\mathrm{mask}}$.
Accordingly, the mean Shapley--Core gaps are negative for the controls,
whereas they are positive for the contextual distractors.

\begin{table}[t]
    \centering
    \caption{Shapley and sampled egalitarian least-core attribution for
    contextual distractors and prediction-active random-SAM controls.  Gap is
    Shapley minus Core, and S $>$ C is the fraction of regions for which the
    Shapley attribution exceeds the Core attribution.}
    \label{tab:natural-shapley-core}
    \begin{tabular}{lllrrrr}
        \toprule
        Model & Region & Outcome & Shapley & Core & Gap & S $>$ C \\
        \midrule
        ResNet-50 & Contextual distractor & $A_D$
          & 2.338 & 1.262 & 1.076 & 82.5\% \\
        & & $A_{\mathrm{mask}}$
          & 1.528 & 0.873 & 0.655 & 78.8\% \\
        \cmidrule(lr){2-7}
        & Random-SAM control & $A_D$
          & 0.614 & 0.838 & $-0.224$ & 33.1\% \\
        & & $A_{\mathrm{mask}}$
          & 0.434 & 0.579 & $-0.144$ & 33.1\% \\
        \addlinespace
        \midrule
        ViT-B/16 & Contextual distractor & $A_D$
          & 2.177 & 1.645 & 0.532 & 80.6\% \\
        & & $A_{\mathrm{mask}}$
          & 1.435 & 1.112 & 0.323 & 76.9\% \\
        \cmidrule(lr){2-7}
        & Random-SAM control & $A_D$
          & 1.810 & 1.959 & $-0.149$ & 40.5\% \\
        & & $A_{\mathrm{mask}}$
          & 1.188 & 1.295 & $-0.107$ & 39.5\% \\
        \bottomrule
    \end{tabular}
\end{table}

These results indicate that the positive Shapley--Core gap is enriched among
regions exhibiting context-dependent sign reversal and is not a generic
property of localized regions that affect the prediction.  
The same pattern appears for both architectures and for both region summaries.

\subsection{Details of the Natural-Image Experiments}
\label{app:natural-image-details}

\subsubsection{Data, preprocessing, and cooperative game}

We use Salient ImageNet class--feature pairs and their corresponding ImageNet
training images.
Both RGB images and Salient ImageNet masks are resized to 256 pixels and
center-cropped to $224\times224$ pixels.
Each image is divided into a $7\times7$ grid, yielding 49 players.
For a coalition $S\subseteq N$, patches outside $S$ are replaced by the fixed
ImageNet RGB mean $(0.485,0.456,0.406)$ before model normalization.
The game value is defined as the target-class logit,
\begin{align}
v(S)=f_c(x_S).
\end{align}
We evaluate torchvision ResNet-50 with \texttt{IMAGENET1K\_V2} weights and
ViT-B/16 with \texttt{IMAGENET1K\_V1} weights.

\subsubsection{Contextual-distractor selection}
\label{app:natural-selection}

\paragraph{Spurious-feature screening.}
Each Salient ImageNet class--feature pair is annotated by five MTurk workers~\cite{salient_imagenet}.
We retain features that a majority of workers judge not to be part of the
main object.
The workers' free-text descriptions are used as textual prompts for the
localization stage described below.

\paragraph{Initial sign-reversal search.}
The Salient ImageNet soft mask is binarized at its image-specific 90th
percentile, and masks covering more than $25\%$ of the image are excluded.
A patch is assigned to a preliminary distractor region $D$ when at least
$20\%$ of its area overlaps the binary mask.

For each candidate region, we sample contexts
$T\subseteq N\setminus D$ by first sampling the context size uniformly and
then sampling a subset of that size uniformly.
We evaluate the region marginal contribution
\begin{align}
\Delta_D(T)
=
v(T\cup D)-v(T).
\end{align}
Let $\tau$ denote the 70th percentile of the sampled $v(T)$ values, and define
\begin{align}
\mathcal{T}_{\mathrm{weak}}
&=
\{T:v(T)<\tau\},
\\
\mathcal{T}_{\mathrm{strong}}
&=
\{T:v(T)\geq\tau\}.
\end{align}
We retain candidates whose mean marginal contribution is positive in weak
contexts and negative in strong contexts,
\begin{align}
\mathbb{E}
\left[
\Delta_D(T)
\mid
T\in\mathcal{T}_{\mathrm{weak}}
\right]
&>0,
\\
\mathbb{E}
\left[
\Delta_D(T)
\mid
T\in\mathcal{T}_{\mathrm{strong}}
\right]
&<0,
\end{align}
while also requiring the corresponding sign to occur in a majority of sampled
contexts in each group.

\paragraph{Independent localization and re-evaluation.}
To obtain an independent localization of the candidate feature, we use
Grounding DINO~\cite{grounding_dino} with the worker-provided textual prompt
and refine the resulting regions using SAM~\cite{sam}.
We retain regions that sufficiently agree with the Salient ImageNet mask and
exclude excessively large masks.
The distractor region $D$ is then redefined from the resulting agreement mask,
and the context-dependent sign-reversal test is repeated.

The same localization procedure is used for both architectures.
For each model, the sign-reversal criterion is evaluated using that model's
own coalition values, allowing the resulting contextual-distractor sets to
differ across architectures.

\subsubsection{Shapley and sampled egalitarian least-core computation}
\label{app:natural-core-computation}

Shapley values are approximated using random player permutations.
For player $i$,
\begin{align}
\widehat{\phi}_i
=
\frac{1}{R}
\sum_{r=1}^{R}
\left[
v(P_i^{(r)}\cup\{i\})
-
v(P_i^{(r)})
\right],
\end{align}
where $P_i^{(r)}$ contains the players preceding $i$ in permutation $r$.

For the least-core approximation, we sample coalition constraints by including
each player independently with probability $0.5$ and explicitly include the
empty and grand coalitions.
We shift the game as
\begin{align}
v'(S)
=
v(S)-v(\emptyset),
\end{align}
and solve
\begin{align}
\min_{x,\varepsilon}
\quad&
\varepsilon
\\
\mathrm{s.t.}\quad
&
\sum_{i\in N}x_i
=
v'(N),
\\
&
\sum_{i\in S}x_i+\varepsilon
\geq
v'(S),
\qquad
S\in\mathcal{C}_{\mathrm{sample}},
\\
&
\varepsilon\geq0.
\end{align}

After obtaining $\varepsilon^\star$, we select an egalitarian least-core
allocation by minimizing
\begin{align}
\|x\|_2^2
\end{align}
under the same coalition and efficiency constraints with
$\varepsilon$ fixed to $\varepsilon^\star$ up to numerical tolerance.
We do not impose nonnegativity on the allocation.
Here and in Section~\ref{sec:natural-contextual-distractors},
\emph{Core} refers to this sampled egalitarian least-core allocation.

\subsubsection{Prediction-active random-SAM controls}
\label{app:random-sam-control}

We construct a separate control cohort for each architecture.
Candidate images are drawn from the same ImageNet pool after excluding the
selected contextual-distractor images.
Randomly placed boxes, with sizes matched to the contextual-distractor region
distribution, are provided to SAM to obtain control regions.
Neither Salient ImageNet masks nor Grounding DINO are used to determine the
control location.

To avoid prediction-irrelevant regions, we retain only controls whose mean
absolute marginal contribution exceeds the 10th percentile of the corresponding
contextual-distractor distribution, yielding thresholds of 0.60 logit units
for ResNet-50 and 0.90 for ViT-B/16.
Regions satisfying the contextual sign-reversal criterion are excluded.
Neither Shapley nor Core attribution is used during control selection.
The selected controls are evaluated using the same Shapley and sampled
least-core procedures as the contextual distractors.

\subsubsection{Scope of the comparison}

The natural contextual-distractor cohort is selected using human annotations,
independent localization, and a context-dependent sign-reversal criterion.
The reported attribution results therefore characterize this selected
population rather than ImageNet images in general.
Furthermore, because the contextual-distractor and random-SAM cohorts are
constructed using different selection criteria, their comparison should be
interpreted as an empirical comparison between the resulting cohorts rather
than as a causal effect of sign reversal alone.

\end{document}